\documentclass[journal]{IEEEtran}
\ifCLASSINFOpdf
\else
\fi

\usepackage[linesnumbered,lined,boxed,vlined]{algorithm2e}

\usepackage[hidelinks]{hyperref}   
\usepackage{hyperref}
\usepackage{subcaption}
\usepackage{epstopdf}
\usepackage{graphicx}
\usepackage{amsfonts}
\usepackage{enumerate}
\usepackage{varioref}
\usepackage{longtable}
\usepackage{graphicx}
\usepackage{amssymb}
\usepackage{chngpage}
\usepackage{multirow}
\usepackage{multirow}
\usepackage{amsthm}
\usepackage{amsmath}
\usepackage{caption}
\usepackage{nccmath}  
\usepackage[table]{xcolor}  
\usepackage{arydshln}  

\usepackage{threeparttable}
\definecolor{bananamania}{rgb}{0.98, 0.91, 0.71}

\usepackage{mathtools}  

\usepackage{amsthm} 
\newtheorem{theorem}{Theorem}

\newtheorem{corollary}{Corollary}

\begin{document}

\title{Tree-Based Optimization: \\A Meta-Algorithm for  Metaheuristic Optimization}

\author{Benyamin Ghojogh$^\dagger$,
		Saeed Sharifian$^{\dagger\dagger}$, Hoda Mohammadzade$^{\dagger\dagger\dagger}$\\

\thanks{$^\dagger$Benyamin Ghojogh is with both Department of Electrical Engineering, Amirkabir University of Technology, Tehran, Iran and Department of Electrical Engineering, Sharif University of Technology, Tehran, Iran. E-mail: \href{mailto:ghojogh_benyamin@ee.sharif.edu}{ghojogh\_benyamin@ee.sharif.edu}}
\thanks{$^{\dagger\dagger}$Saeed Sharifian is with Department of Electrical Engineering, Amirkabir University of Technology, Tehran, Iran. E-mail: \href{mailto:sharifian_s@aut.ac.ir}{sharifian\_s@aut.ac.ir}}
\thanks{$^{\dagger\dagger\dagger}$Hoda Mohammadzade is with Department of Electrical Engineering, Sharif University of Technology, Tehran, Iran. E-mail: \href{mailto:hoda@sharif.edu}{hoda@sharif.edu}}
\thanks{This article was written in spring 2016 and slightly revised in fall 2018.}}


\maketitle

\begin{abstract}
Designing search algorithms for finding global optima is one of the most active research fields, recently. These algorithms consist of two main categories, i.e., classic mathematical and metaheuristic algorithms. This article proposes a meta-algorithm, Tree-Based Optimization (TBO), which uses other heuristic optimizers as its sub-algorithms in order to improve the performance of search.
The proposed algorithm is based on \emph{mathematical tree} subject and improves performance and speed of search by iteratively removing parts of the search space having low fitness, in order to minimize and purify the search space.  
The experimental results on several well-known benchmarks show the outperforming performance of TBO algorithm in finding the global solution. Experiments on high dimensional search spaces show significantly better performance when using the TBO algorithm. The proposed algorithm improves the search algorithms in both accuracy and speed aspects, especially for high dimensional searching such as in VLSI CAD tools for Integrated Circuit (IC) design.
\end{abstract}


\IEEEpeerreviewmaketitle

\section{Introduction}\label{section_introduction}

\IEEEPARstart{R}{ecently} classical and heuristic optimization algorithms are widely used in order to find the global optima in a problem or cost function. Classic mathematical-based optimization algorithms, such as Gradient Descent, are used to find the optima of a function using classic mathematical approaches such as derivation. 
However, metaheuristic optimization algorithms have recently gained attraction for this purpose. The reason may be their soft computing approach for solving difficult optimization problems \cite{talbi2009metaheuristics,yang2010nature,fulcher2008computational}. 
Lots of metaheuristic algorithms are proposed in recent years. Some examples from the first generations of these algorithms are Genetic algorithm \cite{john1992holland,holland1989genetic}, Genetic Programming \cite{john1992koza}, Evolutionary Programming \cite{fogel1966artificial,de1975analysis,koza1990genetic}, Tabu Search \cite{glover1998tabu}, Simulated Annealing \cite{kirkpatrick1983optimization} and Particle Swarm Optimization \cite{kennedy1995j}.

Instances from the recent generations of the metaheuristic optimizers are Ant Lion optimizer \cite{mirjalili2015ant}, Artificial Algae algorithm \cite{uymaz2015artificial}, Binary Bat Algorithm \cite{mirjalili2014binary}, Black Hole algorithm \cite{hatamlou2013black}, Binary Cat swarm optimization \cite{crawford2015binary}, Firefly algorithm \cite{gandomi2013firefly}, Fish swarm algorithm \cite{azad2015solving}, Grey Wolf optimizer \cite{mirjalili2014grey}, Krill Herd algorithm \cite{gandomi2012krill}, Hunting search \cite{oftadeh2009new}, Imperialist Competitive algorithm \cite{atashpaz2007imperialist}, Lion algorithm \cite{yazdani2016lion}, Shuffled Frog-Leaping algorithm \cite{eusuff2006shuffled}, Multi-Verse optimizer \cite{mirjalili2016multi}, and Water Cycle algorithm \cite{sadollah2015water}.

This article proposes a meta-algorithm, originally based on metaheuristic optimization, named Tree Based Optimization (TBO). The TBO algorithm is a meta-algorithm which uses `any' other metaheuristic algorithm as a sub-algorithm.

The remainder of the paper is as follows: Section \ref{section_tree_applications} introduces the mathematical concept of tree and reviews the well-known tree applications and the search methods utilizing trees. The previous work on combination of trees and metaheuristic optimization algorithms is reviewed in Section \ref{section_tree_metaheuristics}. The methodology of the proposed algorithm is presented in Section \ref{section_methodology}. Section \ref{section_analysis} analyzes the proposed algorithm in both time complexity and the required number of iterations for an acceptable solution. The verification of the proposed algorithm is analyzed by experiments in Section \ref{section_experiments}. Finally, Section \ref{section_conclusion} concludes the paper and mentions the possible future direction of this work. 

\begin{figure}[!t]
\centering
\begin{subfigure}[b]{0.2\textwidth}
\centering
\includegraphics[width=1.4in]{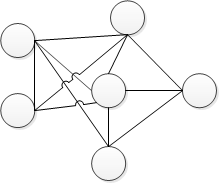} 
\caption{A sample graph}
\label{graph}
\end{subfigure}
\begin{subfigure}[b]{0.2\textwidth}
\centering
\includegraphics[width=1.4in]{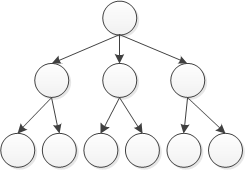}
\caption{A sample tree}
\label{tree}
\end{subfigure}
\caption{Graph and Tree.}
\label{fig:image2}
\end{figure}

\section{Tree and its Popular Applications}\label{section_tree_applications}
To define the tree, graph should be defined first. Graph is a set that consists of several vertices (nodes, or points) connected by some edges. There are two types of graph: (I) Directed graph in which edges have direction, and (II) Regular (or undirected) graph in which edges have no direction. An example of an undirected graph is depicted in Fig. \ref{graph}.

Tree is a graph without any loop. It grows in the same way that the branches of real tree grow. A sample of tree can be seen in Fig. \ref{tree}. 
Tree is a popular mathematical subject which is mostly used for search purposes.
In search problems, the landscape of search (i.e., search space) is represented by a graph.
The goal is to find the shortest path between two specific nodes. 
Different methods of tree search are reviewed briefly in the following.

\begin{figure*}[!t]
\centering
\begin{subfigure}[b]{0.3\textwidth}
\centering
\includegraphics[width=1.2in]{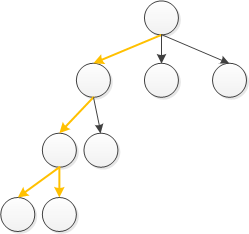} 
\caption{DFS}
\label{DFS}
\end{subfigure}
\begin{subfigure}[b]{0.3\textwidth}
\centering
\includegraphics[width=1.2in]{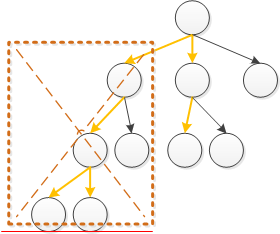}
\caption{DLS}
\label{DLS}
\end{subfigure}
\begin{subfigure}[b]{0.3\textwidth}
\centering
\includegraphics[width=1.2in]{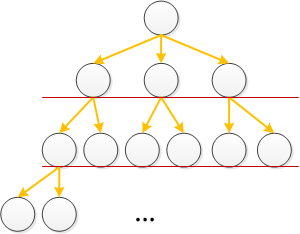}
\caption{BFS}
\label{BFS}
\end{subfigure}
\begin{subfigure}[b]{0.3\textwidth}
\centering
\includegraphics[width=1.5in]{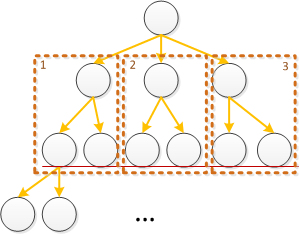}
\caption{IDS}
\label{IDS}
\end{subfigure}
\begin{subfigure}[b]{0.3\textwidth}
\centering
\includegraphics[width=1.2in]{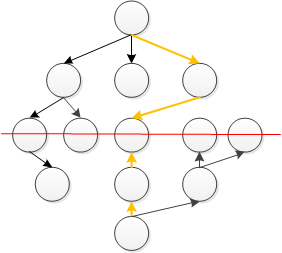}
\caption{BDS}
\label{BDS}
\end{subfigure}
\begin{subfigure}[b]{0.3\textwidth}
\centering
\includegraphics[width=0.9in]{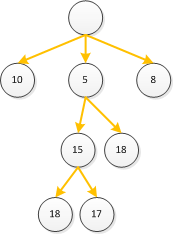}
\caption{UCS}
\label{UCS}
\end{subfigure}
\caption{Different Tree Search methods.}
\label{fig:image2}
\end{figure*}

\begin{figure}[!t]
\centering
\begin{subfigure}[b]{0.49\textwidth}
\centering
\includegraphics[width=2.5in]{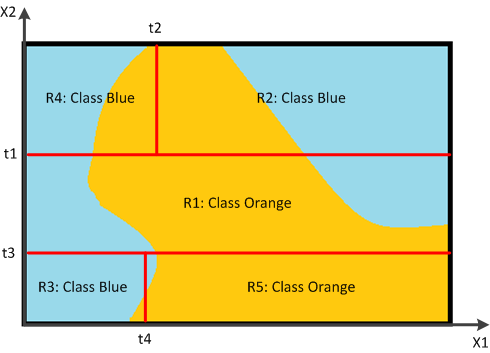} 
\caption{Regions of classification}
\label{tree_classification1}
\end{subfigure}
\bigbreak
\begin{subfigure}[b]{0.49\textwidth}
\centering
\includegraphics[width=1.5in]{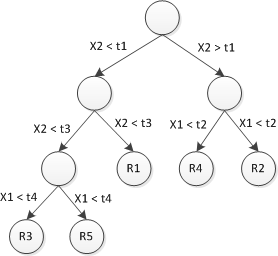}
\caption{The tree of classification}
\label{tree_classification2}
\end{subfigure}
\caption{Classification using tree.}
\label{tree_classification}
\end{figure}

Depth-First Search (DFS) method grows just one of the branches of tree and tries to find the solution by growing that branch as much as necessary, hoping to find it (see Fig. \ref{DFS}). 
One of DFS method's problems is that it may get stuck in an infinite loop if the graph of search has a loop; because it may climb a branch that oscillates between two nodes of graph. In order to solve the problem, the Depth-Limited Search (DLS) is introduced. This method forbids growing the tree more than a specific limit known by the problem. If it reaches the limit, forgets the branch and grows another one (see Fig. \ref{DLS}).

The DLS method may find a solution and get away from lock, however it does not guaranty the minimum (best) solution. Therefore, Breadth-First Search (BFS) is introduced, which grows the tree in breadth and climbs the tree in all possible branches, before growing it once more (see Fig. \ref{BFS}).
In order to have the speed of DLS method and the accuracy of BFS method, Iterative Deepening Search (IDS) is introduced. This method grows tree in one branch to a limit (as in DLS) and then grows it in another branch (as in BFS) which is depicted in Fig. \ref{IDS}.
Another tree search is Bidirectional Search (BDS) which grows two trees, one from the source and one from the destination. Meeting the two trees creates a path between the two nodes (see Fig. \ref{BDS}).

If the edges of the landscape graph has costs, the costs should be considered in search. Uniform-Cost Search (UCS) method associates a cost to each node of a tree. Cost of every tree node is sum of cost of the paths reaching that node (starting from the root). This method grows tree from the node with lowest cost (see Fig. \ref{UCS}) \cite{russell2016artificial,shahhoseini2012evolutionary}.

Another popular application of tree is using tree for regression and classification \cite{friedman2001elements}. In this method, a landscape with dimension of the data is created to be regressed or classified. A tree is applied to this landscape by optimizing the dimension to be split and the splitting point at each node of tree. Hence, the tree tries to divide the landscape into sub-regions in order to solve the classification or regression problem in each sub-region.
It continues growing until the error gets small enough. Another approach is growing the tree to a big one and then start pruning it. Pruning obviously lets error get bigger but it resolves the overfitting problem. An example of classification using tree is illustrated in Fig. \ref{tree_classification}. As is obvious in this figure, the class with the major population is the winner class.
In addition, meta-algorithms such as ADA-Boost usually use tree as their sub-classification algorithm \cite{friedman2001elements}.

\section{Previous Work on Combination of Tree \& Metaheuristics}\label{section_tree_metaheuristics}
There are lots of researches on using tree in metaheuristic algorithms or using metaheuristics in tree and graph problems \cite{blum2011hybrid}.
Branch and bound method in tree search which was first proposed in \cite{land1960automatic} has been recently combined with heuristic methods such as in \cite{nagar1995meta}.
Also, Beam search proposed by \cite{ow1988filtered} is recently combined with Ant Colony (ACO) \cite{dorigo2011ant} which is called the Beam-ACO algorithm \cite{blum2005beam,blum2008beam,lopez2010beam}.

Some researches, such as \cite{applegate2006traveling}, have worked on dividing the problem (e.g. searching) into several sub-problems in order to decrease the complexity of solving the problem. They try to merge the solutions after solving the sub-problems in order to have the solution of the complete problem. In \cite{cotta2003embedding}, the solution merging is proposed as an alternate for crossover operation in evolutionary algorithms.

Another popular approach in tree subject of research is Tree Decomposition, proposed by \cite{robertson1984graph}. Tree decomposition is a tree related to a specific graph, where each vertex of the graph appears in at least one of the nods of it. Every node of this tree consists of a set of vertices of the graph. The connected vertices of graph should be appeared together in some of the sets in nodes \cite{hammerl2015metaheuristic}. 

Lots of metaheuristic algorithms are recently used to perform the tree decomposition. As several examples, Genetic Algorithm \cite{larranaga1997decomposing,musliu2007genetic}, Ant Colony \cite{hammerl2009ant,hammerl2010ant} and iterated Local Search \cite{musliu2007generation,musliu2008iterative} have been used for tree decomposition, in the literature.
In addition, some metaheuristics have used tree decomposition in themselves, instead of using heuristic for tree decomposition, such as \cite{khanafer2012tree,fontaine2011guiding}.

It is noteworthy that some of metaheuristic algorithms such as Ant Colony and Genetic Programming deal with graph and tree subjects \cite{talbi2009metaheuristics}. Ant Colony \cite{drigo1996ant} tries to find the best path in a graph-based problem. Genetic Programming \cite{john1992koza} applies crossover and mutation operators on the set of all possible tree compositions in order to find the solution of a tree-based problem.

\section{TBO Algorithm}\label{section_methodology}
TBO algorithm uses tree in a different approach from other metaheuristic optimization algorithms. It uses tree to minimize the landscape of search (i.e., search space) in order to improve the efficiency and speed of search. This removing is based on finding and omitting the bad parts of landscape which have less fitness compared to better parts with better fitness. This approach will lead to smaller search space and makes the job of search much easier and more accurate. Algorithm will converge to a region of search space which is small enough to search it accurately and is probably the desired solution to the optimization problem.

There are three types of the proposed TBO algorithm, namely binary TBO, multi-branch TBO, and adaptive TBO. These three algorithms are different in several small points but they have the same basic idea. The basics will be covered in the first type and will not be repeated in the others for the sake of brevity.

\subsection{Binary TBO}
Binary TBO uses a binary tree to split the regions of search space. A binary tree is a tree whose every node branches into two edges. TBO algorithm climbs down a branch of tree by dividing the current region of landscape into only two parts, as shown in Fig. \ref{binary_tree_fig}.
Binary TBO algorithm includes different parts explained in the following.

\begin{figure}[!t]
\centering
\begin{subfigure}[b]{0.49\textwidth}
\centering
\includegraphics[width=0.8in]{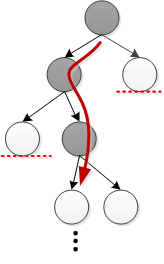} 
\caption{}
\label{binary_tree}
\end{subfigure}
\bigbreak
\begin{subfigure}[b]{0.49\textwidth}
\centering
\includegraphics[width=2in]{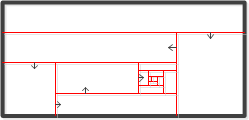}
\caption{}
\label{binary_tree2}
\end{subfigure}
\caption{Climbing down a binary tree.}
\label{binary_tree_fig}
\end{figure}

\subsubsection{Vertical \& Horizontal Split}
To divide the initial or the remaining landscape into two regions, it should be decided whether to split it vertically or horizontally. This decision can be swapped between vertical and horizontal choices, at every iteration of algorithm, formulated as:
\begin{align}
\text{Split Orientation} =
  \left\{
      \begin{array}{ll}
        \text{Vertical}   & \text{if was Horizontal},\\
        \text{Horizontal} & \text{if was Vertical}.\\
      \end{array}
    \right.
\end{align}
Another approach for this step is choosing vertical or horizontal split randomly with a probability $p$ (e.g., $p=0.5$):
\begin{align}
\text{Split Orientation} =
  \left\{
      \begin{array}{ll}
        \text{Vertical} & \text{if } ~ r \in [0,p),\\
        \text{Horizontal} & \text{if } ~ r \in [p,1],\\
      \end{array}
    \right.
\end{align}
where $r$ is a uniform random number in range $[0, 1]$, i.e., $r \sim U(0,1)$.

The binary TBO algorithm can be extended for search spaces with higher dimensions. In high-dimensional search spaces, splitting should be performed in one of the dimensions and the next iteration use another dimension to perform splitting on. In other words, the algorithm splits on dimensions in a pre-determined order in the different iterations.

\subsubsection{Splitting Point Selection}
The next step after deciding the orientation of split is choosing the splitting point in the corresponding region. It can be randomly one of the possible points in the region; however, it is better to choose a point in a limited range, slightly far away from the edges of the remaining region, to avoid splitting into a very wide and a very small region. If the the remaining region to split is $[R_1,R_2]$, then:
\begin{equation}
\label{eqn1}
\text{Splitting point} \gets U(L_1,L_2),
\end{equation}
where $L_1 = (30\% \times (R_2 - R_1)) + R_1$ and $L_2 = (70\% \times (R_2 - R_1)) + R_1$. The numbers $30\%$ and $70\%$ are hyper-parameters, and $U(\alpha,\beta)$ is a uniform random number in the range [$\alpha$,$\beta$].

\subsubsection{Performing the Sub-Algorithm}
As was already mentioned, TBO is a meta-algorithm which uses other metaheuristic algorithms as sub-algorithms. In this step of TBO algorithm, the sub-algorithm (e.g. Genetic Algorithm, Particle Swarm Optimization, simple Local search, etc) is performed to search in each of the two regions, separately.

It is also possible to use different sub-algorithms in (I) different iterations, (II) the two divided regions, or (III) according to the latest best solutions and/or size of each region.

\subsubsection{Updating global best found so far}
Almost all metaheuristic algorithms store the best solution in a memory and so does the TBO algorithm. After searching the two divided regions, each region outputs a best solution of itself. If any of these best solutions outperforms the global best found so far, the last stored global best will be replaced with it, formulated as:
\begin{align}
\text{GB}_{i} =
  \left\{
      \begin{array}{ll}
        \text{B}_i & \text{if} ~~ \text{B}_i > \text{GB}_{i-1},\\
        \text{GB}_{i-1} & \text{otherwise},\\
      \end{array}
    \right.
\end{align}
where $\text{B}_i$ is the best of the best solutions of the two regions in the $i$-th iteration and GB indicates the global best found so far.

\subsubsection{Entering a Region with a Probability}
This step can be considered as the bottleneck of the algorithm. Deciding a region to enter into is the most important decision of this algorithm because the other region will be removed from the search space and will not be searched any time unless the algorithm gets another run after termination.

At the first glance, it seems that choosing the region with better found solution is reasonable. However a clever approach does not accept this definite decision because of the possibility of trapping in local optimum solutions. Notice that, the best solutions of regions may be the local solutions and far away from the desired global best. In other words, the global best might even be in the region with the worse solution found in this iteration and the sub-algorithm has not found it in that region!
To avoid this possible problem, a probabilistic decision should be employed.
However, probabilistic decision causes another problem. What if the decision approach chooses a region in which there is not the desired global best. If this happens, the algorithm will converge to an outlier (unknown) or local solution! This problem is accepted, but it was sacrificed to avoid trapping in a local optimum solution. To overcome this problem, the algorithm can be run for several times to make sure about the answer. Note that this re-running approach is common in metaheuristic algorithms because of their randomness manner and un-sureness.

Setting the probability of entering every region is related to the best solutions found in each region by the sub-algorithms. The better the best solution of a region in the iteration, the more probable it is for entering into. This approach somehow prevents the problem of converging to an outlier but still does not ensure absolute correctness.

To formulate the probabilistic decision, the type of problem, which is minimization or maximization, matters. It is assumed that the problem is to reduce a cost function so the minimum of a function is required (otherwise the cost function can be multiplied by $-1$). In such situation, the decision is made as:
\begin{align}
\text{P}_1 =
  \left\{
      \begin{array}{ll}
        1 - \frac{\text{B}_1}{\text{B}_1 + \text{B}_2} & \text{if} \quad \text{B}_1 > 0 ~ \& ~ \text{B}_2 > 0,\\
        \frac{\text{B}_1}{\text{B}_1 + \text{B}_2} & \text{if} \quad \text{B}_1 < 0 ~ \& ~ \text{B}_2 < 0,\\
        1 - \frac{\text{BB}_1}{\text{BB}_1 + \text{BB}_2} & \text{otherwise},\\
      \end{array}
    \right.
\end{align}
where $\text{B}_1$ and $\text{B}_2$ are the best solutions of regions 1 and 2 in this iteration, respectively. 
$\text{P}_1$ is the probability of entering region 1.
$\text{BB}_1$ and $\text{BB}_2$ are also:
\begin{equation*}
\label{eqn1}
\text{BB}_1 = \text{B}_1 + |\text{min}(\text{B}_1,\text{B}_2)| + 1,
\end{equation*}

\begin{equation}
\label{eqn1}
\text{BB}_2 = \text{B}_2 + |\text{min}(\text{B}_1,\text{B}_2)| + 1,
\end{equation}
in order to make both of them positive and compare them easily.
The decision of entering a region is then computed as:
\begin{align}
\text{Enter to} =
  \left\{
      \begin{array}{ll}
        \text{Region 1} & \text{if} ~~ r < \text{P}_1,\\
        \text{Region 2} & \text{if} ~~ r \geq \text{P}_1,\\
      \end{array}
    \right.
\end{align}
where $r$ is a uniform random number. i.e., $r \sim U(0,1)$.

After entering the selected region, the other region is removed from the search space and the boundaries of the landscape are updated as:
\begin{align}
\label{eqn1}
  \left\{
      \begin{array}{ll}
        y_{min} = \text{SP} & ~ \text{if horizontal split, entering up}\\
        y_{max} = \text{SP} & ~ \text{if horizontal split, entering down}\\
        x_{min} = \text{SP} & ~ \text{if vertical split, entering right}\\
        x_{max} = \text{SP} & ~ \text{if vertical split, entering left}\\
      \end{array}
    \right.
\end{align}
where $\text{SP}$ is the splitting point.

\subsubsection{[Optional] Adaptive Number of Sub-Algorithm Particles, According to Iteration}
The number of particles can be reduced in an adaptive way as the algorithm goes forward because by improving the algorithm, the search space becomes smaller and smaller and there is no need to have as many particles as in the first iterations with the whole search space. Although this may cause a small drop in accuracy, but it improves the speed of algorithm especially in huge problems. If the step of reduction is denoted by integer value $\Delta_1$ (e.g. is one), the number of particles can be updated as:
\begin{equation}
\label{eqn1}
\text{NP}_i \gets \text{max}(\text{NP}_{i-1} - \Delta_1, \, \text{a lower bound}),
\end{equation}
where $\text{NP}_i$ is the number of particles in the $i$-th iteration.

\subsubsection{[Optional] Adaptive Number of Sub-Algorithm Particles, According to Region Size}
In the previous section, the number of particles of the sub-algorithm was changed as the algorithm proceeds. This number can also be changed according to the size of the divided regions. Consider two regions separated by a splitting line. One of them is much larger than the other. It is obvious that the larger region needs more particles to search almost the entire parts of it, though the other one can be searched by fewer particles. Hence, the number of particles can be updated as:
\begin{align}
& \text{NP}_1 \gets \text{round}(\text{NPP} \times \frac{\text{size}_1}{\text{size}_1 + \text{size}_2}),\\
& \text{NP}_2 \gets \text{round}(\text{NPP} \times \frac{\text{size}_2}{\text{size}_1 + \text{size}_2}),
\end{align}
where $\text{NP}_i$, NPP, and $\text{size}_i$ are the number of particles in the $i$-th region, the total number of particles in both regions, and the size of $i$-th region, respectively.

\begin{figure}[!t]
\centering
\includegraphics[width=2in]{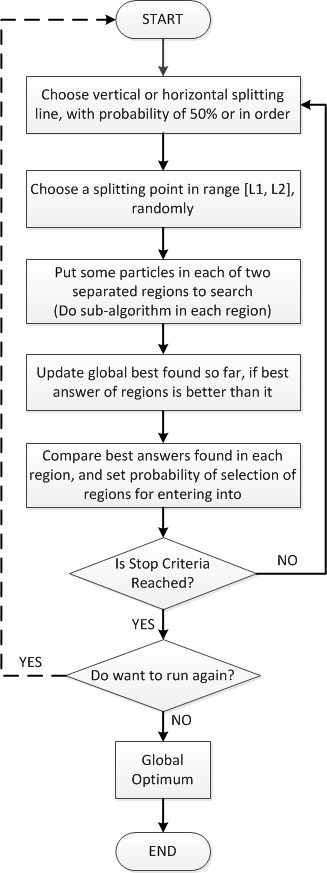}
\caption{Binary TBO flowchart.}
\label{TBO_flowchart}
\end{figure}

\subsubsection{[Optional] Adaptive Number of Sub-Algorithm Iterations}
As the algorithm proceeds, the size of the search space becomes smaller. Therefore, the number of global iterations of the sub-algorithm to be run in the two regions can be reduced in an adaptive way to improve the speed of the algorithm. It can be formulated as:
\begin{equation}
\label{eqn1}
\text{SI}_i \gets \text{max}(\text{SI}_{i-1} - \Delta_2, \, \text{a lower bound}),
\end{equation}
where $\text{SI}_i$ is the number of sub-algorithm iterations in the $i$-th iteration and $\Delta_2$ is an integer denoting the step of reduction (e.g. is one). 

\subsubsection{Termination}
As mentioned before, the algorithm might converge to an outlier. It rarely occurs but might happen in difficult landscapes with lots of local optimum solutions and a global optima hard to find.
In order to be sure about the answer, the algorithm can be run for several times. This re-running approach is common in metaheuristic algorithms because of their randomness manner and un-sureness.

To terminate one run of TBO algorithm, it should have done the splitting for several defined times indicating the depth of the tree. The depth of tree, i.e., the number of iterations of a run of TBO algorithm, should be chosen according to the size of search landscape and the required accuracy rate.

\subsubsection{Algorithm flowchart}
As an abstract to the explained details, The flowchart of binary TBO algorithm is illustrated in Fig. \ref{TBO_flowchart}.

\subsubsection{Experimental result}
To examine the proposed binary TBO, the \emph{Schaffer} benchmark (see Table \ref{benchmarks_table} and Fig. \ref{benchmarks_figure}) is tested. As is shown in Fig. \ref{binary_test1}, the algorithm has split the landscape in binary regions at every step and finally converged to the global minimum. In this test, the used sub-algorithm is Genetic Algorithm (GA) \cite{john1992holland,holland1989genetic} and the number of particles (chromosomes) of sub-algorithm is five (in each region). The depth of tree is also chosen to be 10.

As already was mentioned, the algorithm may converge to an outlier motivating us to run it for several times. Though it happens rarely but one of these happenings is depicted in Fig. \ref{binary_test2}.

\begin{figure}[!t]
\centering
\begin{subfigure}[b]{0.5\textwidth}
\centering
\includegraphics[width=3in]{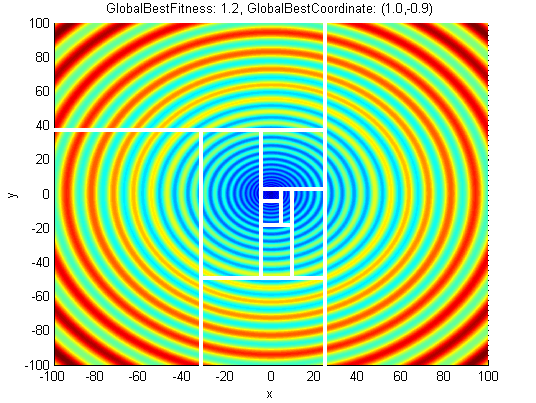} 
\caption{Converging to the global minimum}
\label{binary_test1}
\end{subfigure}
\bigbreak
\begin{subfigure}[b]{0.5\textwidth}
\centering
\includegraphics[width=3in]{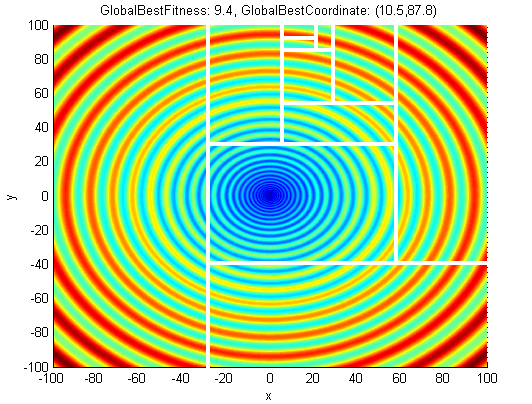}
\caption{Converging to an outlier}
\label{binary_test2}
\end{subfigure}
\caption{Binary TBO experiment.}
\label{binary_test}
\end{figure}

\subsection{Multi-branch TBO}
It may be better if the splitting step gets some changes according to the dimensionality of the search space, denoted by $d$. Since the dimensionality of the search space might be huge in real experiments, binary splitting may not satisfy the complexity of the problem and the possibility of convergence to an outlier increases. In order to consider the dimensionality of search space and its complexity, the number of splittings in each iteration of TBO algorithm is recommended to be equal to the dimensionality of the search space.

The multi-branch TBO algorithm uses a tree with a branching factor $\alpha$, meaning that every node splits into $\alpha$ branches.
The $\alpha$ is an even number and is obtained as follows:
\begin{equation}
\label{eqn1}
\alpha = 2^d,
\end{equation}
which means the number of splittings in every iteration is equal to $d$ (every dimension is split).

This algorithm splits the remaining landscape into $\alpha$ partitions and performs the sub-algorithm in each of them, at every iteration. In other words, a splitting line is created for each dimension and these lines are perpendicular to each other. The splitting into $\alpha$ parts is actually climbing down the multi-branch tree as is shown in Fig. \ref{multi_branch} for a two-dimensional search space.

\begin{figure}[!t]
\centering
\includegraphics[width=1.5in]{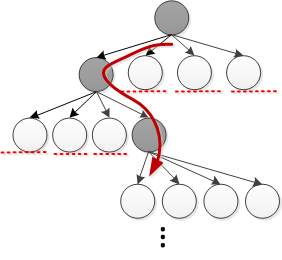}
\caption{Climbing down a multi-branch Tree.}
\label{multi_branch}
\end{figure}

\begin{figure}[!t]
\centering
\includegraphics[width=3in]{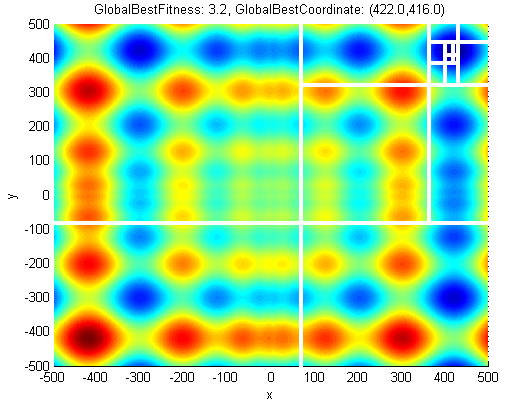}
\caption{Multi-branch TBO experiment.}
\label{multi_test}
\end{figure}

The settings and steps of multi-branch TBO is similar to binary TBO with some small modifications. For example, setting the probability of entering a region is formulated as:
\begin{align}
\text{P}_i =
  \left\{
      \begin{array}{ll}
        1 - \frac{\text{B}_i}{\sum_{\text{regions}} \text{B}_j} & \text{if} ~~ \text{B}_j > 0, ~~ \forall \text{ regions},\\
        \frac{\text{B}_i}{\sum_{\text{regions}} \text{B}_j} & \text{if} ~~ \text{B}_j<0, ~~ \forall \text{ regions},\\
        1 - \frac{\text{BB}_i}{\sum_{\text{regions}} \text{BB}_j} & \text{otherwise}\\
      \end{array}
    \right.
\end{align}
\begin{align}
& \text{BB}_i = \text{B}_i + |\min_{\forall j}(\text{BB}_j)| + 1,\\
& \text{sort}(\text{P}_i ~ ; ~ \text{descending}),
\end{align}
\begin{align}
\text{Enter to} =
  \left\{
      \begin{array}{ll}
        \text{Region }m & \text{if} \quad U(0,1) < \text{P}_m,\\
        \text{Region }n & \text{else if} \quad U(0,1) < \text{P}_{n},\\
        \vdots\\
      \end{array}
    \right.
\end{align}
where $\text{P}_m > \text{P}_n > ...$, which had been sorted.

In Multi-branch TBO for a two-dimensional search space, removing the three regions to which the algorithm has not entered is performed as:
\begin{equation}
\label{eqn1}
  \left\{
      \begin{array}{l}
      	\left.
      	\begin{array}{l}
        y_{min} = \text{SP}_{\text{vertical}}\\
        x_{min} = \text{SP}_{\text{horizontal}}\\
        \end{array}
        \right\} \text{entering up-right}
        \\
        \left.
      	\begin{array}{l}
        y_{min} = \text{SP}_{\text{vertical}}\\
        x_{max} = \text{SP}_{\text{horizontal}}\\
        \end{array}
        \right\} \text{entering up-left}
        \\
        \left.
      	\begin{array}{l}
        y_{max} = \text{SP}_{\text{vertical}}\\
        x_{min} = \text{SP}_{\text{horizontal}}\\
        \end{array}
        \right\} \text{entering down-right}
        \\
        \left.
      	\begin{array}{l}
        y_{max} = \text{SP}_{\text{vertical}}\\
        x_{max} = \text{SP}_{\text{horizontal}}\\
        \end{array}
        \right\} \text{entering down-left}
        \\
      \end{array}
    \right.
\end{equation}

An experiment of multi-branch TBO is performed for the \emph{Schwefel} benchmark (see Table \ref{benchmarks_table} and Fig. \ref{benchmarks_figure}). In this test, the sub-algorithm is Particle Swarm Optimization (PSO) algorithm \cite{kennedy1995j} with five particles (for each region in which sub-algorithm is run). The depth of tree is chosen to be five and the branching factor of each node of tree is four because the dimensionality of landscape is two. The result of experiment is shown in Fig. \ref{multi_test}. As is obvious in this figure, The TBO has removed the parts of landscape with high costs and converges to the global minimum. It has escaped from the local minimum solutions as expected.

\subsection{Adaptive TBO}\label{section_adaptive_TBO}
The TBO algorithm can be adaptive in branching factor $\alpha$. This parameter can be adapted by multiple causes. It can be adapted by the best fitness of the region which is going to be split. The better the fitness of region to be split, the larger the branching factor can get, for it is more worthful to be searched and is worthy to run more sub-algorithms in it. 

Another approach for adaptation is making algorithm adaptive by the size of region. The larger the size of region which is going to be split, the larger the branching factor is worthy to get because it has more space required to be searched.

There is also another approach for adaptation. As the algorithm improves and iterations go forward, the branching factor can get smaller since at the first iterations, the landscape is large requiring more search.

All the mentioned adaptations can be applied on both binary TBO and multi-branch TBO.
An experiment of adaptive TBO is done here in which adaptation is applied on binary TBO and is according to the iteration number. The total number of iterations (depth of tree) is set to be 10 and branching factor ($\alpha$) of the first iteration is 4, the second is 3 and the other iterations have $\alpha = 2$ as in Binary TBO. 
This adaptive TBO example is climbing down a tree with four, three, and two branches in the first, second, and other levels of its depth, respectively (see Fig. \ref{adaptive}).
The result is shown in Fig. \ref{adaptive_test}.

\begin{figure}[!t]
\centering
\includegraphics[width=1.5in]{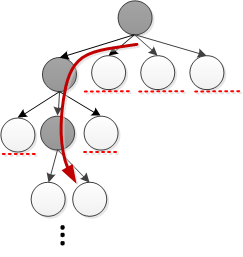}
\caption{Climbing down the tree of adaptive TBO.}
\label{adaptive}
\end{figure}

\begin{figure}[!t]
\centering
\includegraphics[width=3in]{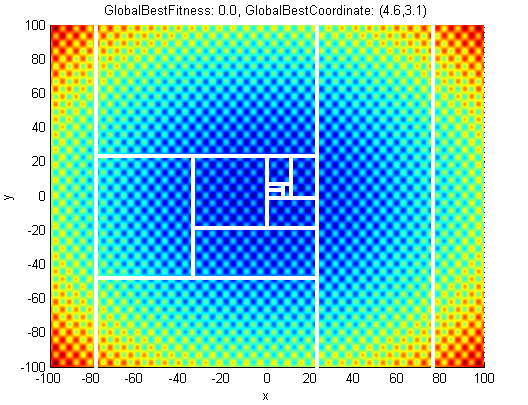}
\caption{Adaptive TBO experiment.}
\label{adaptive_test}
\end{figure}

As an abstraction to all the details explained and described in previous sections, the pseudo-code of TBO algorithm can be seen in Algorithm \ref{TBO_algorithm}. 


\SetAlCapSkip{0.5em}
\IncMargin{0.8em}
\begin{algorithm}[h]
\DontPrintSemicolon
    \textbf{START:}
	\textbf{Initialize} parameters\;
	\While{Stop criteria is not reached}{
	    number of particles $\gets$ number of particles $- \Delta_1$\;
	    sub-algorithm iter. $\gets$ sub-algorithm iter. $- \Delta_2$\;
	    $\alpha \propto$ ``Size of region'' or `` Iter. number''\;
	    \uIf{Multi-Branching}{
	        \For{$j$ from $1$ to $d$}{
	            \For{$i$ from $1$ to $\alpha$}{
	                splitting point $\gets$ a point $\in [L_1, L_2]$
	            }
	        }
	    }
	    \Else{
	        \textbf{Swap} state of split on dimensions (e.g., vertical or horizontal)\;
	        \For{$i$ from $1$ to $\alpha$}{
	            splitting point $\gets$ a point $\in [L_1, L_2]$
	        }
	    }
	    \ForAll{regions}{
	        number of particles $\gets$ size of region\;
	        \textbf{Do} the optimization sub-algorithm
	    }
	    Best of iteration $\gets$ Best of best of regions\;
	    \If{Best of iteration is better than Global best}{
	        Global best $\gets$ Best of iteration
	    }
	    \ForAll{regions}{
	        \textbf{Set} Probability of entering the region
	    }
	    Enter which region? $\gets$ Probability $U(0, 1)$\;
	    \textbf{Remove} outer region from search space
	}
	\uIf{Want to run again}{
	    \textbf{Goto} START of procedure
	}
	\Else{
	    \textbf{Return} the Global best
	}
\caption{TBO Algorithm}\label{TBO_algorithm}
\end{algorithm}
\DecMargin{0.8em}

\section{Analysis of TBO Algorithm}\label{section_analysis}

\subsection{Time Complexity of TBO Algorithm}

\begin{theorem}
The time complexity of TBO algorithm is $\mathcal{O}(k)$ where $k$ is the depth of tree (or number of splits) if the sub-algorithm is considered to be an oracle with time complexity $\mathcal{O}(1)$.
\end{theorem}
\begin{proof}
The tree is grown to the depth $k$ in all the binary, multi-branch, and the adaptive TBO algorithms. Assuming that the sub-algorithm is an oracle having time complexity $\mathcal{O}(1)$, the TBO algorithm runs for $\mathcal{O}(k)$ times.
\end{proof}

\begin{table*}[!t]
\begin{minipage}{\textwidth}
\renewcommand{\arraystretch}{2}  
\caption{Benchmarks used for test}
\label{benchmarks_table}
\centering
\begin{tabular}{l l l c}
\hline
\hline
\multicolumn{1}{c}{\textbf{Name of Benchmark}} & \multicolumn{1}{c}{\textbf{Function}} & \multicolumn{1}{c}{\textbf{Range}} & \multicolumn{1}{c}{\textbf{Step}} \\
\hline
\textbf{F1:} Sphere & $f(x) = \sum_{i=1}^d x_i^2$ & $x_i \in [-100, 100]$ & $0.1$\\
\textbf{F2:} Griewank & $f(x) = \frac{1}{4000} \sum_{i=1}^d x_i^2 - \prod_{i=1}^d \text{cos}(\frac{x_i}{\sqrt{i}}) + 1$ & $x_i \in [-100, 100]$ & $0.1$\\
\textbf{F3:} Schaffer & $f(x) = \sum_{i=1}^{d-1} (x_i^2 + x_{i+1}^2)^{0.25} [\text{sin}^2 (50 \times (x_i^2 + x_{i+1}^2)^{0.1}) + 1.0]$ & $x_i \in [-100, 100]$ & $0.1$\\
\textbf{F4:} Schwefel & $f(x) = 418.982\,d - \sum_{i=1}^d x_i \text{sin} (\sqrt{|x_i|})$ & $x_i \in [-500, 500]$ & $1$\\
\hline
\hline
\end{tabular}
\end{minipage}
\end{table*}

\subsection{Analysis of Tree Depth \& Algorithm Progress}

\begin{theorem}\label{theorem_depth_binary}
Assuming that all the selected regions include the global optima, the required depth of tree (or the number of splits) in binary TBO algorithm is $k = -d\,(1+\log_2(\varepsilon))$ for having a probability of success greater than $p = 1 - \varepsilon$ for reaching the optima.
\end{theorem}
\begin{proof}
In binary TBO, every branching splits the search space into two regions. On average, the expected value of the split point is at the half of the remaining region resulting in $\frac{1}{2}$ space remained in every splitting. Therefore, after $k$ splits (tree depth $k$), the remained region has area or volume equal to $(\underbrace{\frac{1}{2} \times \dots \times \frac{1}{2}}_{k \text{ times}}) = (\frac{1}{2})^k$. This area or volume (of the last remained region) is required to be $(2\,\varepsilon)^d$ finally if we consider a $d$-dimensional cube having side length $(2\,\varepsilon)$ around the optima. To better explain, we consider a rectangular neighborhood around the optima at its center where the sides have distance $\varepsilon$ from the center. Thus, after $k$ number of splits, we have $(2\,\varepsilon)^d = (\frac{1}{2})^k$ resulting in $k = -d\,(1+\log_2(\varepsilon))$. Note that $\varepsilon$ is a small positive number so the $k$ is a positive number.
\end{proof}

\begin{theorem}\label{theorem_depth_multibranch}
Assuming that all the selected regions include the global optima, the required depth of tree (or the number of splits) in multi-branch TBO algorithm is $k = -(1+\log_2(\varepsilon))$ for having a probability of success greater than $p = 1 - \varepsilon$ for reaching the optima.
\end{theorem}
\begin{proof}
The proof is similar to the proof of the previous theorem except that at every split, $2^d$ regions are divided resulting in the last remained area or volume $(\underbrace{\frac{1}{2^d} \times \dots \times \frac{1}{2^d}}_{k \text{ times}}) = (\frac{1}{2^d})^k$. Having $(2\,\varepsilon)^d = (\frac{1}{2^d})^k$ results in $k = -(1+\log_2(\varepsilon))$.
\end{proof}

\begin{figure}[!t]
\centering
\begin{subfigure}[b]{0.23\textwidth}
\centering
\includegraphics[width=1.5in]{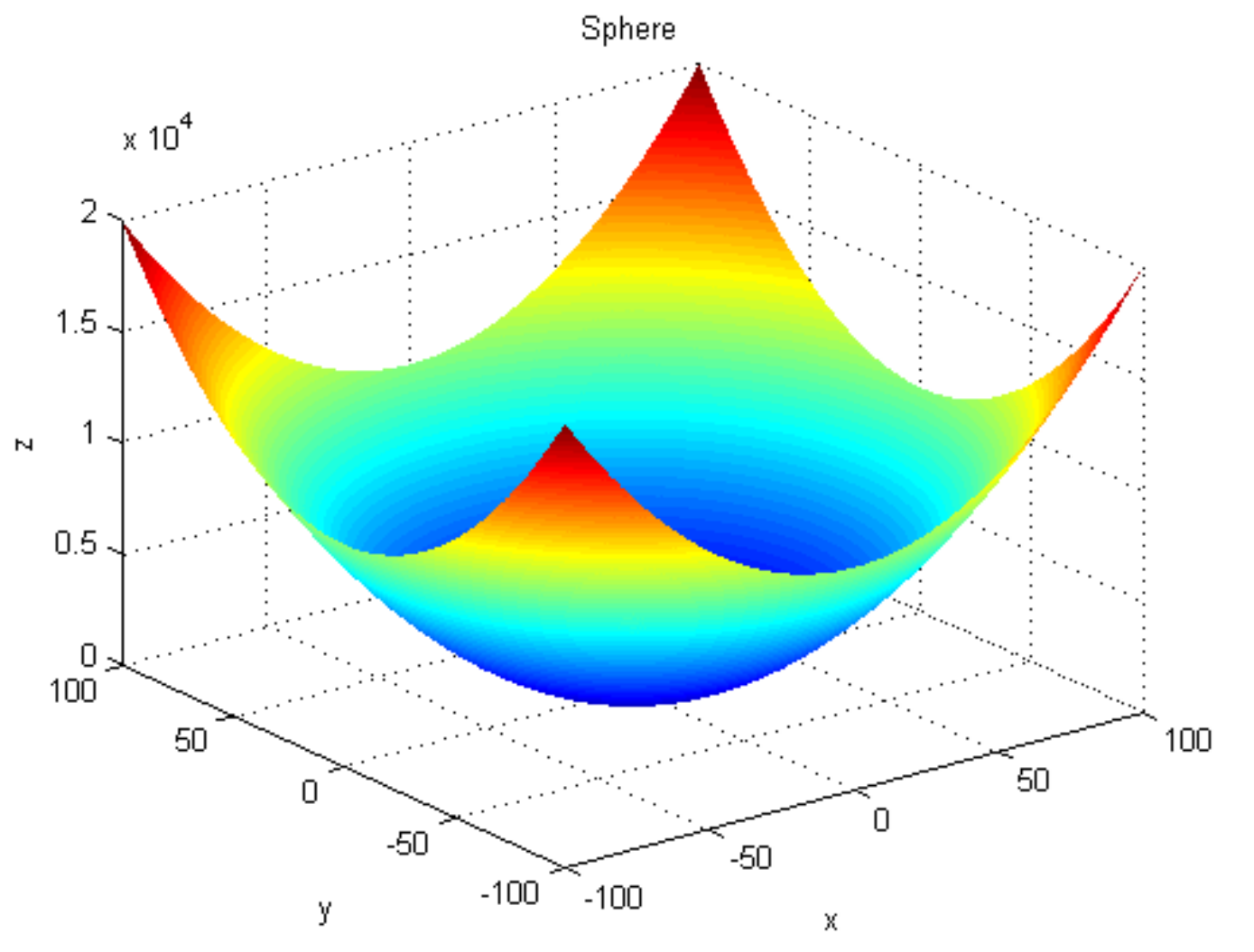} 
\caption{Sphere}
\label{fig:subim1}
\end{subfigure}
\begin{subfigure}[b]{0.23\textwidth}
\centering
\includegraphics[width=1.5in]{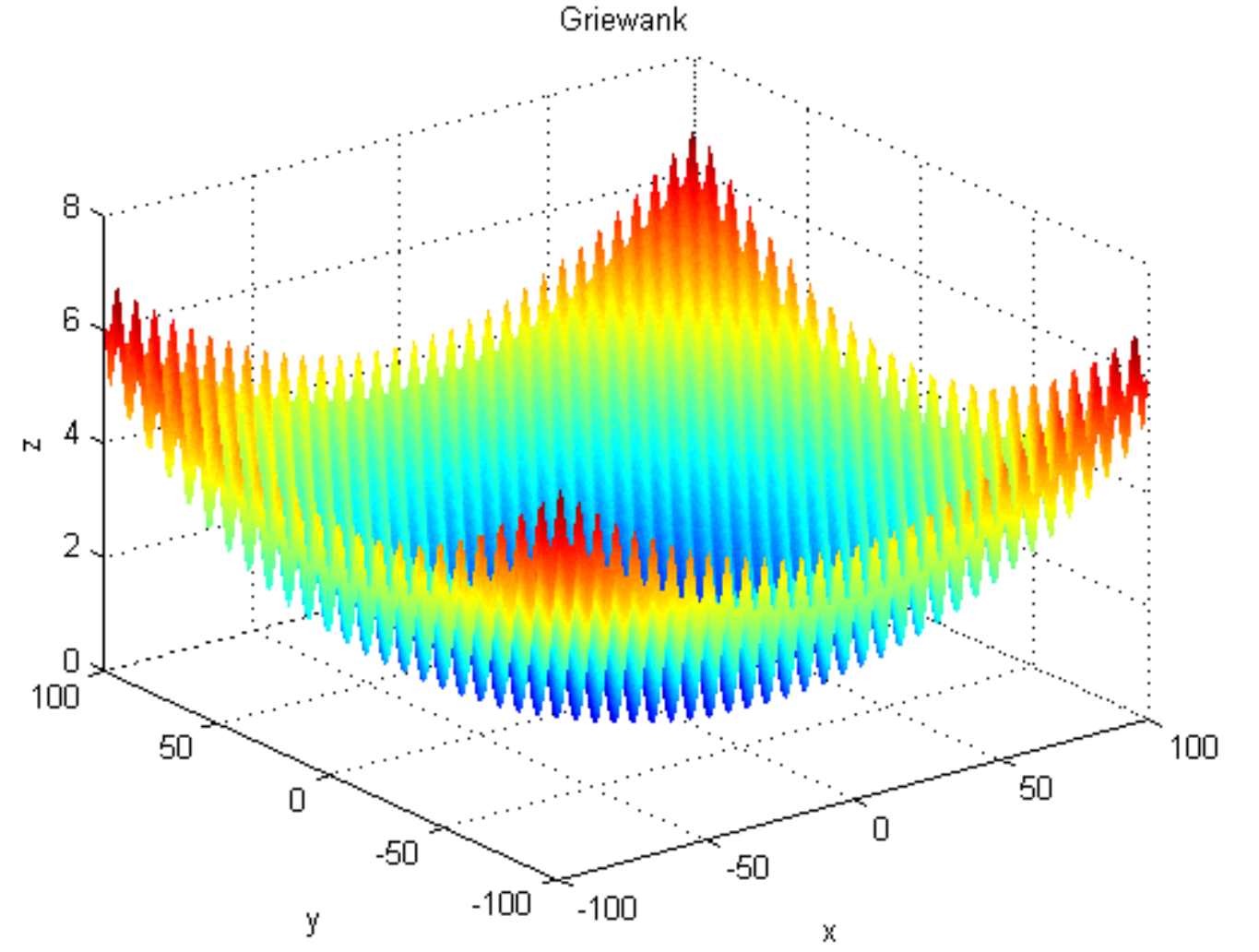}
\caption{Griewank}
\label{fig:subim2}
\end{subfigure}
\bigbreak
\begin{subfigure}[b]{0.23\textwidth}
\centering
\includegraphics[width=1.5in]{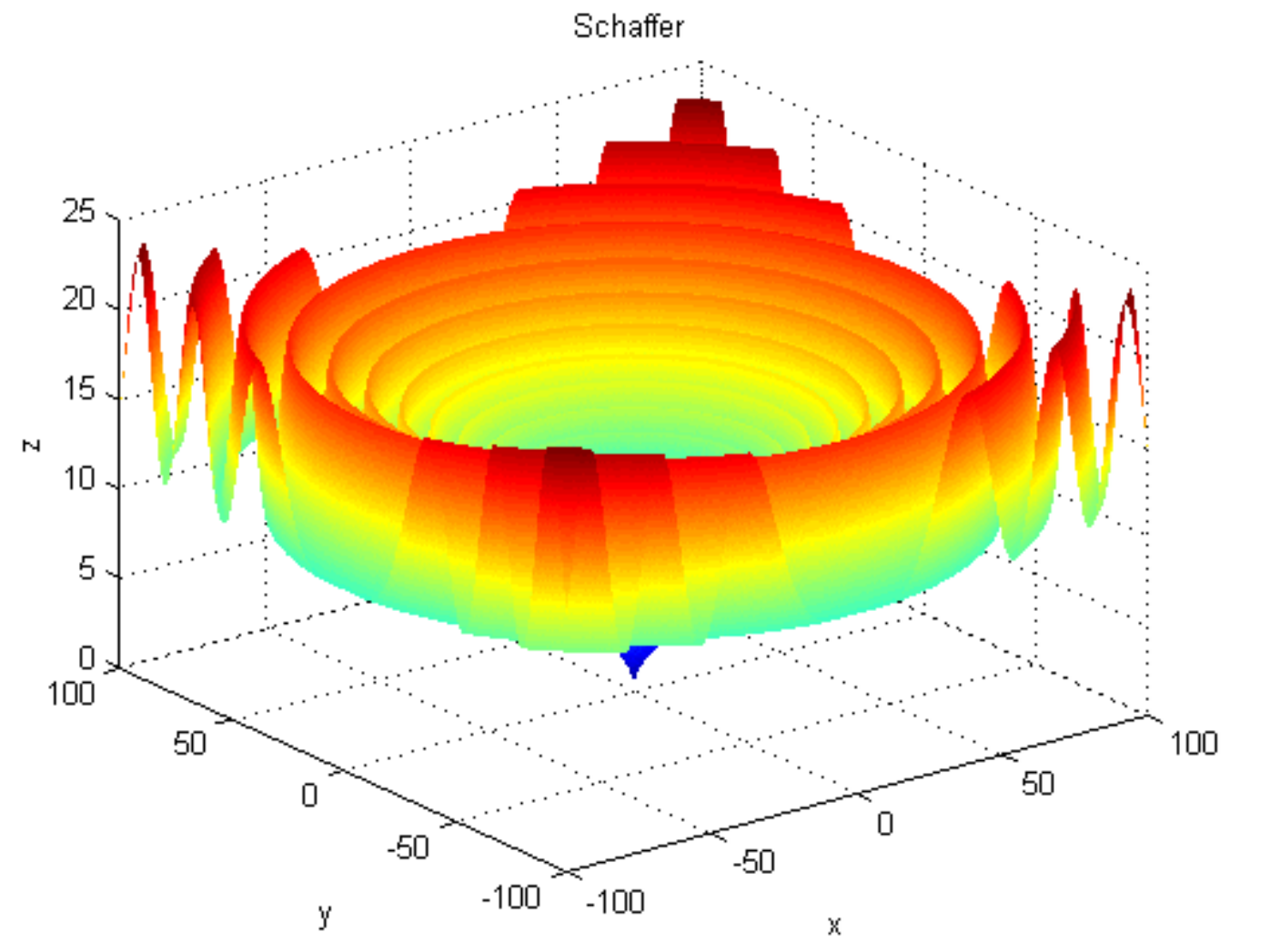}
\caption{Schaffer}
\label{fig:subim2}
\end{subfigure}
\begin{subfigure}[b]{0.23\textwidth}
\centering
\includegraphics[width=1.5in]{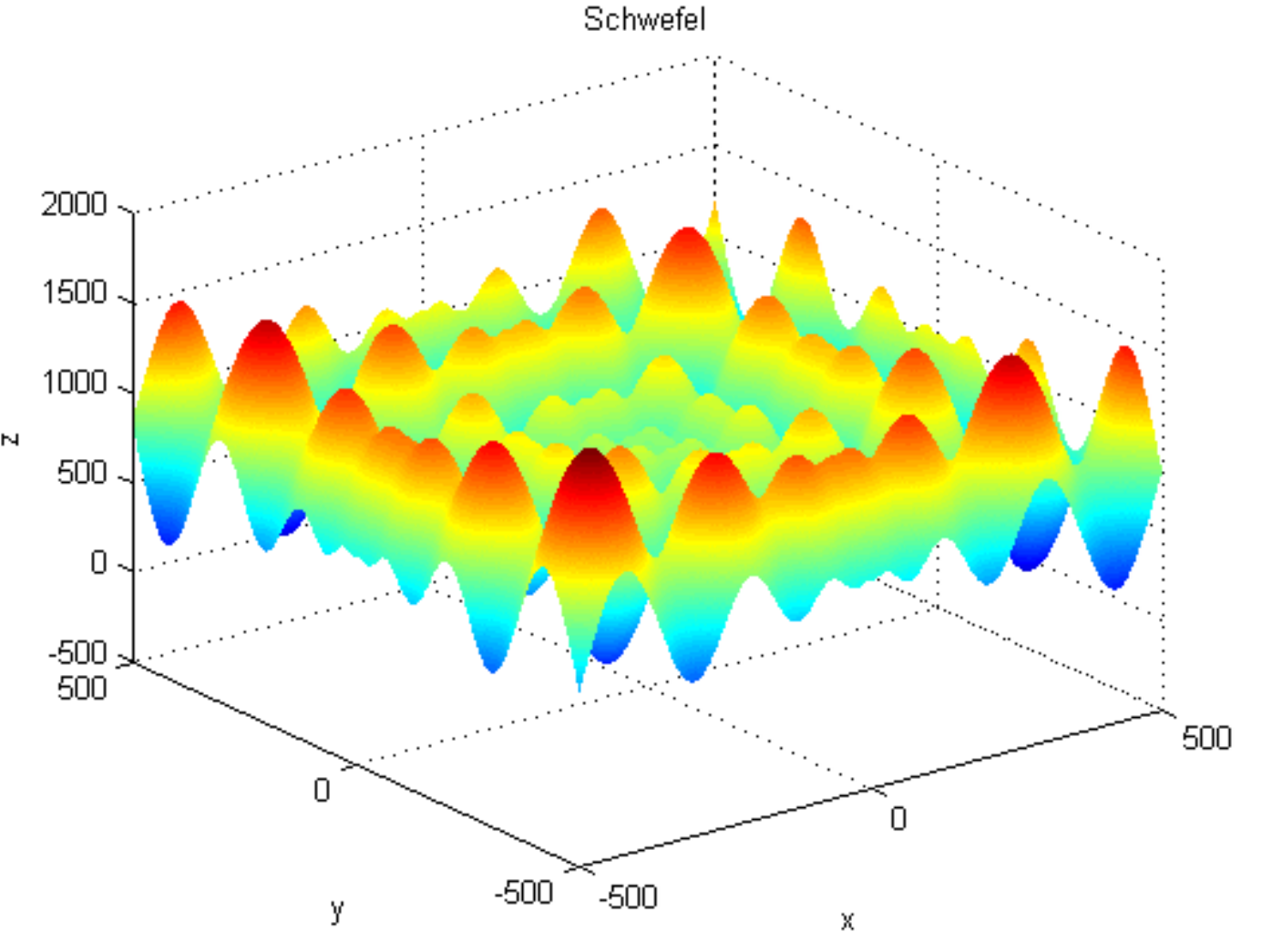}
\caption{Schwefel}
\label{fig:subim2}
\end{subfigure}
\caption{Benchmarks of Table \ref{benchmarks_table}.}
\label{benchmarks_figure}
\end{figure}

\begin{table}[!t]
\renewcommand{\arraystretch}{1.3}  
\caption{Average errors of the found best answer from actual best answer in two dimensional experiments}
\label{table_test}
\centering
\scalebox{0.79}{    
\begin{tabular}{l | l | l | c c c}
\hline
\hline
\multicolumn{3}{c|}{} & LS & PSO & GA\\
\hline
\multirow{8}{*}{ \textbf{F1: Sphere}} 
& \multirow{2}{*}{Binary TBO} 
& 5 Particles & 6.25\% & 0.01\% & 0.23\%\\
& & 10 Particles & 1.52\% & 0.01\% & 0.07\%\\
\cline{2-6}
& \multirow{2}{*}{Multi-branch TBO} 
& 5 Particles & 2.51\% & 0.00\% & 0.09\%\\
& & 10 Particles & 1.53\% & 0.00\% & 0.07\%\\
\cline{2-6}
& \multirow{2}{*}{Adaptive TBO} 
& 5 Particles & 3.81\% & $\simeq$ 0\% & 0.06\%\\
& & 10 Particles & 1.36\% & $\simeq$ 0\% & 0.04\%\\
\cline{2-6}
& \multirow{2}{*}{Not using TBO} 
& 5 Particles & 0.67\% & 0.19\% & 0.29\%\\
& & 10 Particles & 0.38\% & 0.19\% & 0.06\%\\
\hline
\hline
\multirow{8}{*}{ \textbf{F2: Griewank}} 
& \multirow{2}{*}{Binary TBO} 
& 5 Particles & 9.20\% & 1.63\% & 6.19\%\\
& & 10 Particles & 3.50\% & 0.91\% & 4.18\%\\
\cline{2-6}
& \multirow{2}{*}{Multi-branch TBO} 
& 5 Particles & 3.78\% & 0.30\% & 2.51\%\\
& & 10 Particles & 4.81\% & 0.15\% & 2.12\%\\
\cline{2-6}
& \multirow{2}{*}{Adaptive TBO} 
& 5 Particles & 7.21\% & 1.17\% & 4.40\%\\
& & 10 Particles & 3.98\% & 1.04\% & 2.72\%\\
\cline{2-6}
& \multirow{2}{*}{Not using TBO} 
& 5 Particles & 5.13\% & 2.61\% & 6.84\%\\
& & 10 Particles & 4.14\% & 2.24\% & 4.38\%\\
\hline
\hline
\multirow{8}{*}{ \textbf{F3: Schaffer}} 
& \multirow{2}{*}{Binary TBO} 
& 5 Particles & 5.79\% & 0.13\% & 1.10\%\\
& & 10 Particles & 2.85\% & 0.11\% & 0.40\%\\
\cline{2-6}
& \multirow{2}{*}{Multi-branch TBO} 
& 5 Particles & 2.21\% & 0.01\% & 0.32\%\\
& & 10 Particles & 4.15\% & $\simeq$ 0\% & 0.57\%\\
\cline{2-6}
& \multirow{2}{*}{Adaptive TBO} 
& 5 Particles & 2.45\% & 0.05\% & 0.15\%\\
& & 10 Particles & 1.74\% & 0.08\% & 0.18\%\\
\cline{2-6}
& \multirow{2}{*}{Not using TBO} 
& 5 Particles & 0.80\% & 0.24\% & 0.75\%\\
& & 10 Particles & 0.45\% & 0.29\% & 0.16\%\\
\hline
\hline
\multirow{8}{*}{ \textbf{F4: Schwefel}} 
& \multirow{2}{*}{Binary TBO} 
& 5 Particles & 75.81\% & 36.86\% & 20.87\%\\
& & 10 Particles & 47.72\% & 20.48\% & 0.10\%\\
\cline{2-6}
& \multirow{2}{*}{Multi-branch TBO} 
& 5 Particles & 73.06\% & 32.76\% & 0.11\%\\
& & 10 Particles & 51.02\% & 12.29\% & 0.14\%\\
\cline{2-6}
& \multirow{2}{*}{Adaptive TBO} 
& 5 Particles & 63.06\% & 16.38\% & 0.10\%\\
& & 10 Particles & 25.51\% & 4.09\% & 0.09\%\\
\cline{2-6}
& \multirow{2}{*}{Not using TBO} 
& 5 Particles & 5.28\% & 95.01\% & 53.30\%\\
& & 10 Particles & 0.68\% & 76.82\% & 26.80\%\\
\hline
\hline
\end{tabular}%
}
\end{table}

\begin{table}[!t]
\renewcommand{\arraystretch}{1.3}  
\caption{Average errors of using TBO algorithm vs. not using TBO algorithm in three dimensional experiments}
\label{table_test_3D}
\centering
\scalebox{0.79}{    
\begin{tabular}{l | l | c c}
\hline
\hline
\multicolumn{2}{c|}{} & 5 Particles & 10 Particles\\
\hline
\multirow{2}{*}{ \textbf{F1: Sphere}} 
& Multi-branch TBO + GA & 0.58\% & 0.58\% \\
\cline{2-4}
& Just GA & 0.99\% & 0.71\% \\
\hline
\hline
\multirow{2}{*}{ \textbf{F2: Griewank}} 
& Multi-branch TBO + GA & 7.62\% & 0.28\% \\
\cline{2-4}
& Just GA & 13.04\% & 5.50\% \\
\hline
\hline
\multirow{2}{*}{ \textbf{F3: Schaffer}} 
& Multi-branch TBO + GA & $\simeq$ 0\% & 0.40\% \\
\cline{2-4}
& Just GA & 0.88\% & 2.66\% \\
\hline
\hline
\multirow{2}{*}{ \textbf{F4: Schwefel}} 
& Multi-branch TBO + GA & 0.37\% & 0.74\% \\
\cline{2-4}
& Just GA & 3.65\% & 1.63\% \\
\hline
\hline
\hline
\multirow{2}{*}{ \textbf{Total Average Error}} 
& Multi-branch TBO + GA & \textbf{2.14\%} & \textbf{0.50\%} \\
\cline{2-4}
& Just GA & 4.64\% & 2.62\% \\
\hline
\hline
\end{tabular}%
}
\end{table}

\begin{corollary}
The multi-branch TBO algorithm converges $d$ times faster than the binary TBO.
\end{corollary}
\begin{proof}
According to Theorem \ref{theorem_depth_binary} and Theorem \ref{theorem_depth_multibranch}, it can be seen that the binary TBO algorithm needs a depth, $d$ times larger than the depth required by multi-branch TBO for convergence. This means the number splits and thus the number of iterations are $d$ times larger in binary TBO compared to multi-branch TBO. 
\end{proof}

Regarding the adaptive TBO, as any different adaptation can be chosen and the number of splits in every iteration might be any value determined by the rules of programmer, we do not analyze its required tree depth. However, the analysis for adaptive TBO is similar to the approaches of Theorem \ref{theorem_depth_binary} and Theorem \ref{theorem_depth_multibranch}.

\begin{figure*}[!t]
\centering
\begin{subfigure}[b]{0.3\textwidth}
\centering
\includegraphics[width=2.2in]{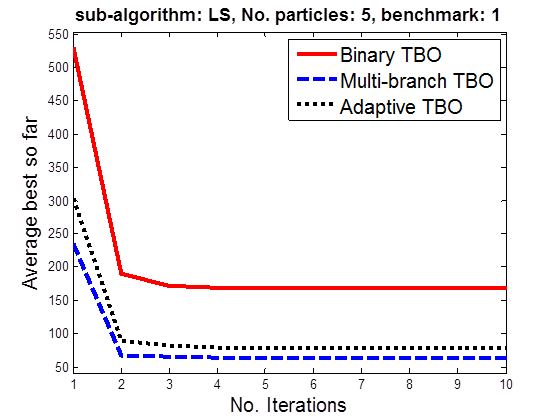} 
\caption{}
\label{fig:subim1}
\end{subfigure}
\begin{subfigure}[b]{0.3\textwidth}
\centering
\includegraphics[width=2.2in]{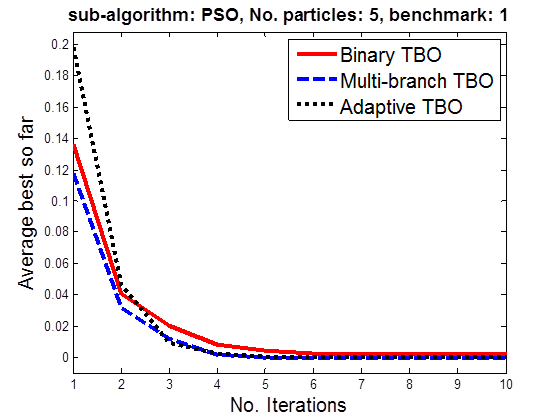} 
\caption{}
\label{fig:subim1}
\end{subfigure}
\begin{subfigure}[b]{0.3\textwidth}
\centering
\includegraphics[width=2.2in]{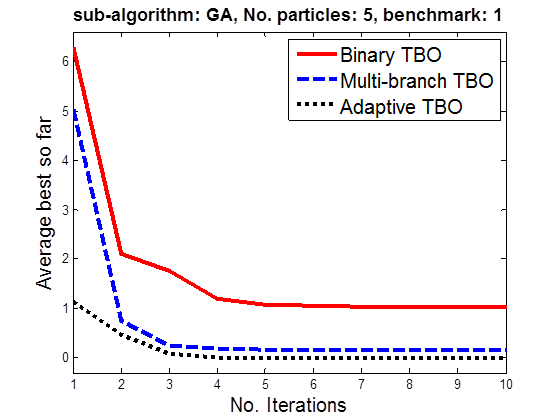} 
\caption{}
\label{fig:subim1}
\end{subfigure}
\begin{subfigure}[b]{0.3\textwidth}
\centering
\includegraphics[width=2.2in]{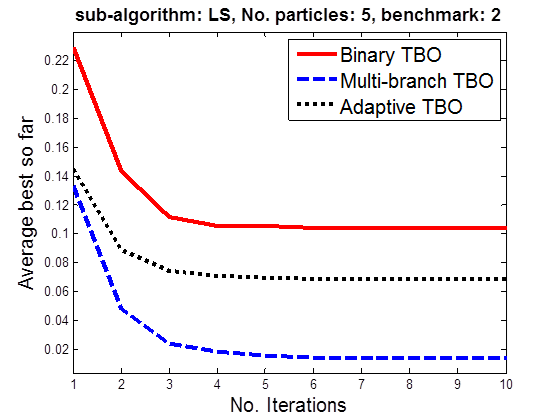} 
\caption{}
\label{fig:subim1}
\end{subfigure}
\begin{subfigure}[b]{0.3\textwidth}
\centering
\includegraphics[width=2.2in]{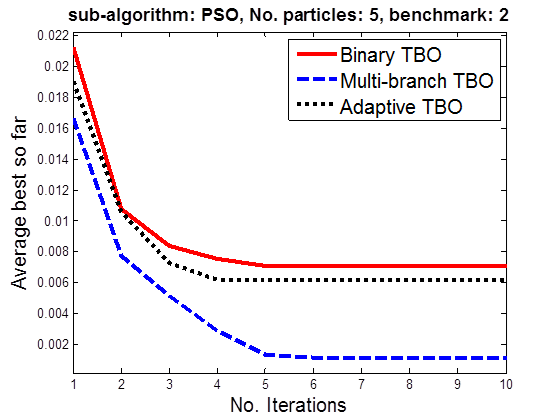} 
\caption{}
\label{fig:subim1}
\end{subfigure}
\begin{subfigure}[b]{0.3\textwidth}
\centering
\includegraphics[width=2.2in]{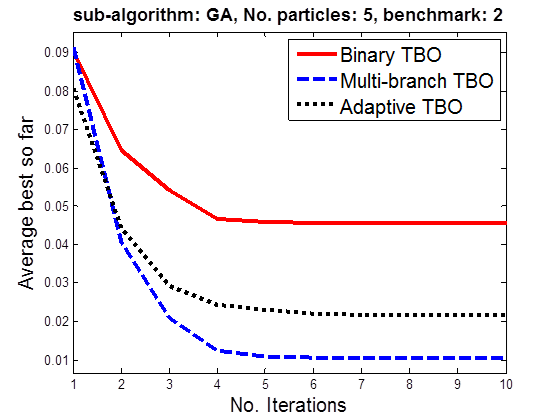} 
\caption{}
\label{fig:subim1}
\end{subfigure}
\begin{subfigure}[b]{0.3\textwidth}
\centering
\includegraphics[width=2.2in]{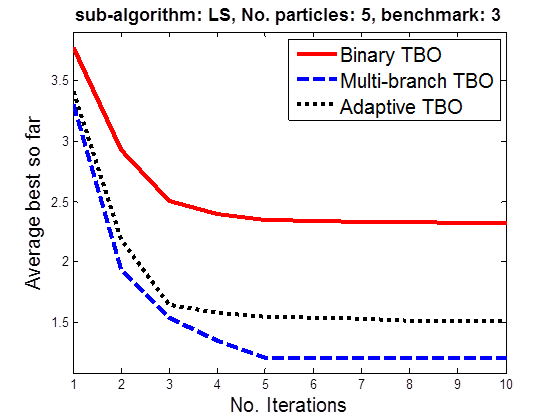} 
\caption{}
\label{fig:subim1}
\end{subfigure}
\begin{subfigure}[b]{0.3\textwidth}
\centering
\includegraphics[width=2.2in]{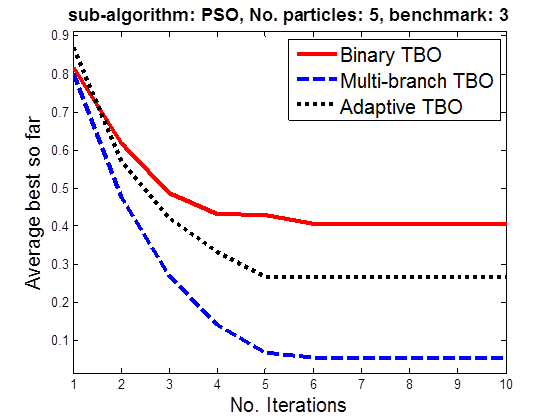} 
\caption{}
\label{fig:subim1}
\end{subfigure}
\begin{subfigure}[b]{0.3\textwidth}
\centering
\includegraphics[width=2.2in]{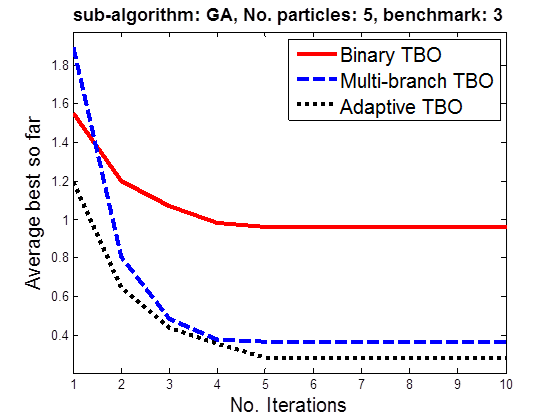} 
\caption{}
\label{fig:subim1}
\end{subfigure}
\begin{subfigure}[b]{0.3\textwidth}
\centering
\includegraphics[width=2.2in]{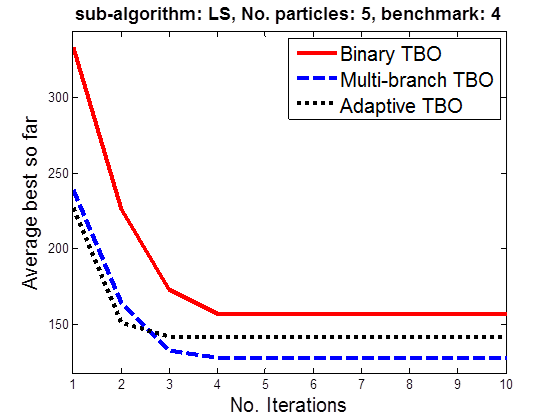} 
\caption{}
\label{fig:subim1}
\end{subfigure}
\begin{subfigure}[b]{0.3\textwidth}
\centering
\includegraphics[width=2.2in]{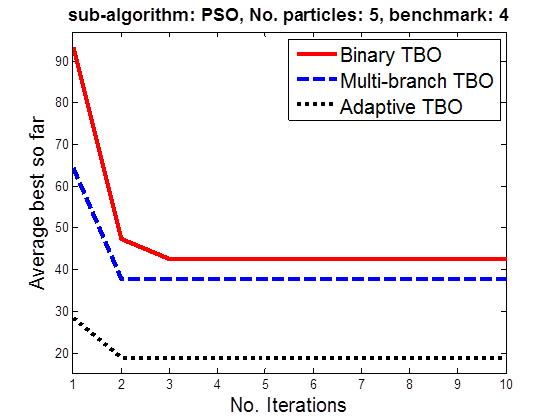} 
\caption{}
\label{fig:subim1}
\end{subfigure}
\begin{subfigure}[b]{0.3\textwidth}
\centering
\includegraphics[width=2.2in]{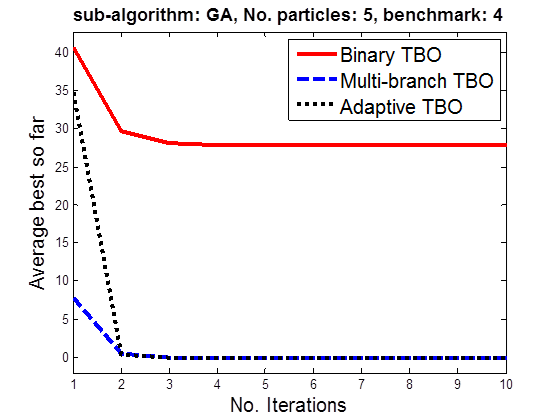} 
\caption{}
\label{fig:subim1}
\end{subfigure}
\caption{Comparison of binary, multi-branch, and adaptive TBO in different benchmarks and sub-algorithms with 5 particles.}
\label{test_5_particles}
\end{figure*}

\section{Experimental Results}\label{section_experiments}
\subsection{Two Dimensional Experiments}
To test the algorithm, the three proposed types of TBO algorithm, i.e., binary, multi-branch and adaptive TBO, are examined. Three sub-algorithms are used for each of them, which are Local Search (LS), Particle Swarm Optimization (PSO) \cite{kennedy1995j} and Genetic Algorithm (GA) \cite{john1992holland,holland1989genetic}. Local Search is locating particles randomly in the landscape to search locally. 
For PSO, global PSO algorithm is used which takes the best answer found so far as the global best and the best
answer found by each particle, as the local best of that particle. In GA, the selection of parents for generating the next generation is chosen to be proportional selection.

The experiments are done for both populations of 5 and 10 particles. The adaptation in the adaptive TBO is the same as the test mentioned in Section \ref{section_adaptive_TBO}. The depth of tree in all experiments is set to be 10.
The benchmarks used for test are listed in Table \ref{benchmarks_table} and shown in Fig. \ref{benchmarks_figure}.
The results of the experiments are shown in figures \ref{test_5_particles} and \ref{test_10_particles} for 5 and 10 particles, respectively. Each of these plots is the average of 25 times of experiments. The average errors of experiments are listed in Table \ref{table_test} and the normalized average errors of this table are depicted in Fig. \ref{bar}.
As is obvious in this figure and was expected, multi-branch and adaptive TBO algorithms perform better than binary TBO in all experiments. Multi-branch TBO performs better than adaptive TBO in some of benchmarks and in some cases, adaptive TBO is better. In addition, Table \ref{table_test} clarifies that more particles causes better search as was expected.

\begin{figure*}[!t]
\centering
\begin{subfigure}[b]{0.3\textwidth}
\centering
\includegraphics[width=2.2in]{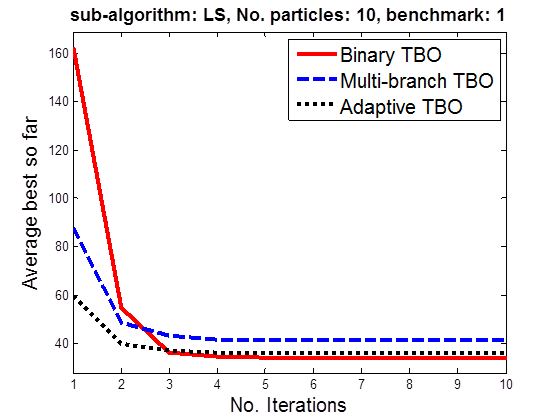} 
\caption{}
\label{fig:subim1}
\end{subfigure}
\begin{subfigure}[b]{0.3\textwidth}
\centering
\includegraphics[width=2.2in]{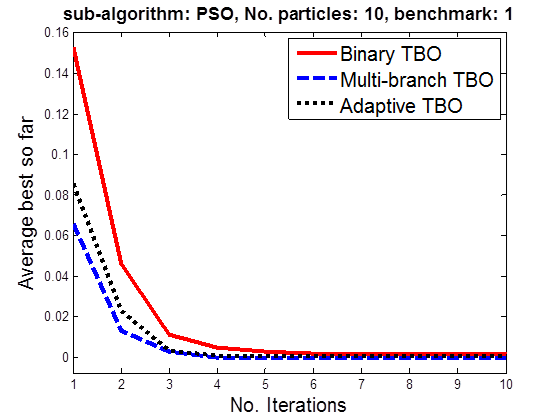} 
\caption{}
\label{fig:subim1}
\end{subfigure}
\begin{subfigure}[b]{0.3\textwidth}
\centering
\includegraphics[width=2.2in]{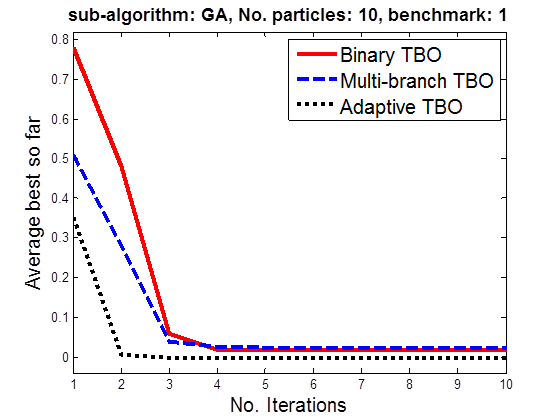} 
\caption{}
\label{fig:subim1}
\end{subfigure}
\begin{subfigure}[b]{0.3\textwidth}
\centering
\includegraphics[width=2.2in]{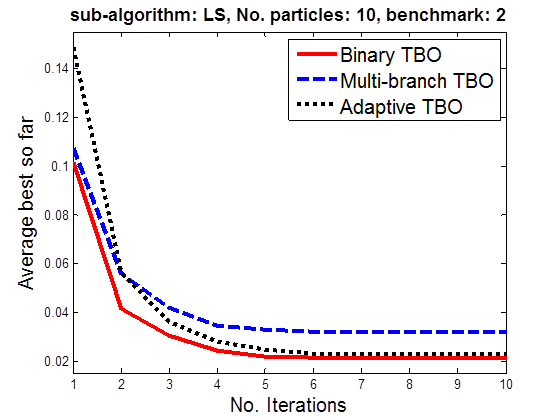} 
\caption{}
\label{fig:subim1}
\end{subfigure}
\begin{subfigure}[b]{0.3\textwidth}
\centering
\includegraphics[width=2.2in]{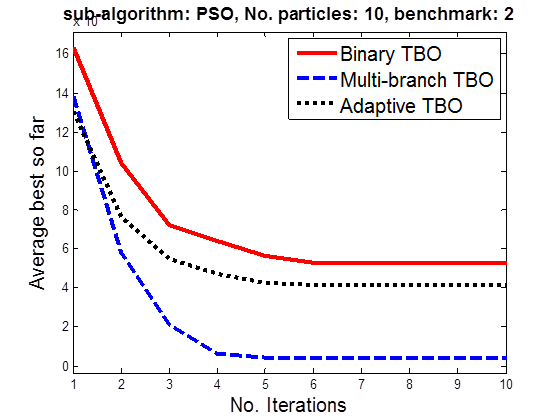} 
\caption{}
\label{fig:subim1}
\end{subfigure}
\begin{subfigure}[b]{0.3\textwidth}
\centering
\includegraphics[width=2.2in]{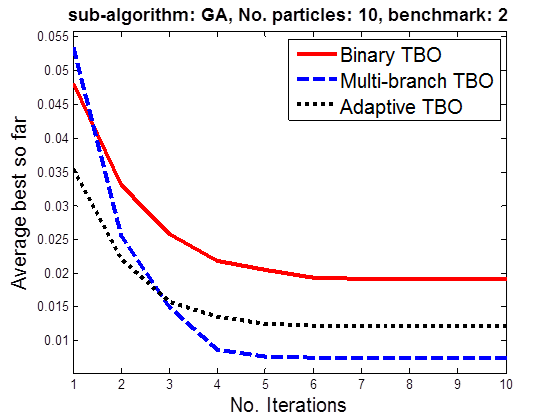} 
\caption{}
\label{fig:subim1}
\end{subfigure}
\begin{subfigure}[b]{0.3\textwidth}
\centering
\includegraphics[width=2.2in]{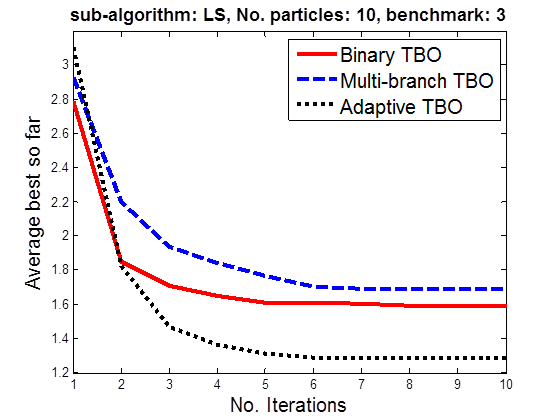} 
\caption{}
\label{fig:subim1}
\end{subfigure}
\begin{subfigure}[b]{0.3\textwidth}
\centering
\includegraphics[width=2.2in]{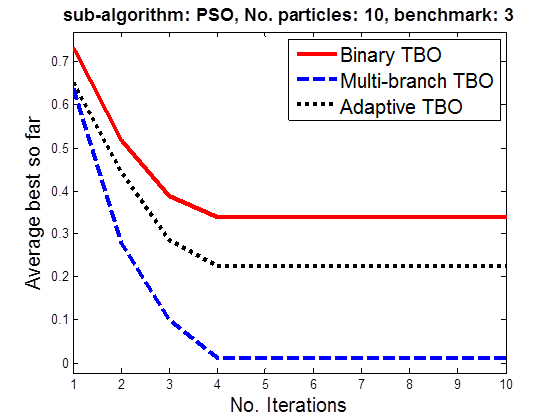} 
\caption{}
\label{fig:subim1}
\end{subfigure}
\begin{subfigure}[b]{0.3\textwidth}
\centering
\includegraphics[width=2.2in]{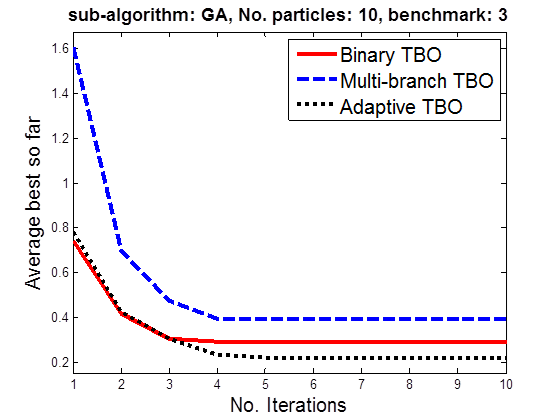} 
\caption{}
\label{fig:subim1}
\end{subfigure}
\begin{subfigure}[b]{0.3\textwidth}
\centering
\includegraphics[width=2.2in]{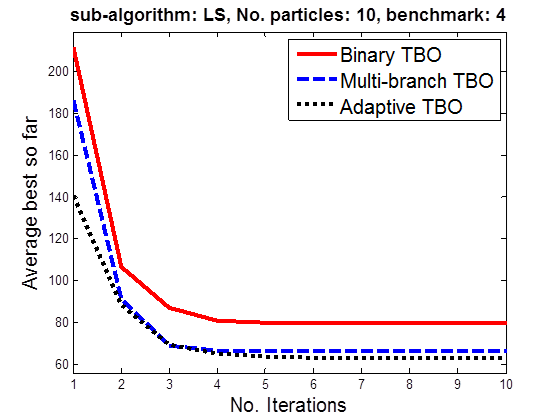} 
\caption{}
\label{fig:subim1}
\end{subfigure}
\begin{subfigure}[b]{0.3\textwidth}
\centering
\includegraphics[width=2.2in]{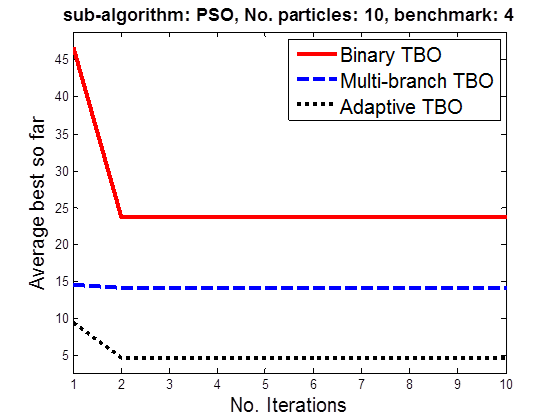} 
\caption{}
\label{fig:subim1}
\end{subfigure}
\begin{subfigure}[b]{0.3\textwidth}
\centering
\includegraphics[width=2.2in]{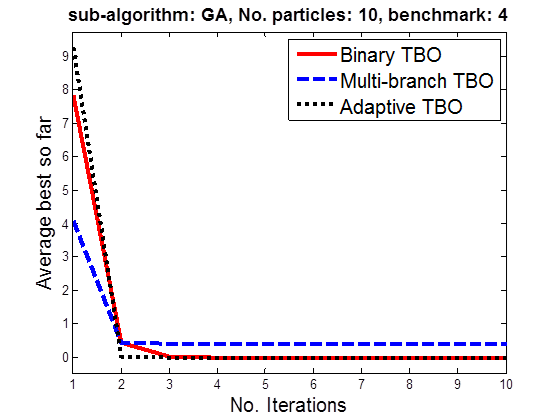} 
\caption{}
\label{fig:subim1}
\end{subfigure}
\caption{Comparison of binary, multi-branch, and adaptive TBO in different benchmarks and sub-algorithms with 10 particles.}
\label{test_10_particles}
\end{figure*}

It was not possible to plot both results of using TBO algorithms and not using TBO (using just sub-algorithms) in one figure because the essence and the number of iterations in the meta-algorithm and the sub-algorithm are totally different. That is why figures \ref{test_5_particles} and \ref{test_10_particles} depict only results of running TBO algorithm.

To show better performance of TBO algorithm in comparison to using the optimization algorithms lonely, the results of experiment are compared with alone local search, PSO and GA in Table \ref{table_test} and Fig. \ref{bar}. Benchmarks 2 and 3 are searched better using TBO algorithm. Although TBO algorithm seems to fail to lonely search algorithms in benchmarks 1 and 4 but Table \ref{table_test_3D} shows absolute better performance of it in all benchmarks for higher dimensional cases of search.

\subsection{Three Dimensional Experiments}

It is claimed that TBO algorithm shows its huge performance better in high dimensions because it removes the bad parts of landscape in order to guide searching to the region containing global optima. An example of high dimensional search is Integrated Circuit (IC) designing in Very Large Scale Integration (VLSI) Computer-Aided Design (CAD) tools which use search algorithms to find an optimum solution for power, surface, delay, and capacitance of chip's die. 

In order to prove the claim of better performance in high dimensional landscapes, an experiment is performed by multi-branch TBO and GA as its sub-algorithm. Table \ref{table_test_3D} compares using multi-branch TBO and GA with using just GA in the four mentioned benchmarks. Note that the benchmarks are set to have three dimensions that cannot be depicted because of their fourth dimension of value. The limit of all benchmarks is set to be $x_i \in [-100,100]$ with the step of one.
As can be seen in Table \ref{table_test_3D} and Fig. \ref{bar_for_3D}, TBO algorithm outperforms using just sub-algorithm for searching in all 3D experiments. 

If this wonderful result is obtained for three dimensional search, this algorithm will perform greatly much better in very high dimensional search since as the dimension gets higher, the complexity of search grows exponentially and the TBO algorithm (especially multi-branch TBO) can help searching become more accurate and also faster. As a conclusion, this algorithm can be utilized for faster and more accurate search especially in high dimensional search such as VLSI IC designing.

Notice that although multi-branch TBO is used for testing the high dimensional search, binary and adaptive TBO can also be used in such problems. For example, binary TBO splits the high dimensional landscape in every dimension, one by one, for each iteration.

\subsection{Some discussions}

As is obvious for high dimensional search, reducing the landscape of search, improves the performance and speed of search exponentially. That is why TBO algorithm outperforms using just sub-algorithms in high dimensional searches.

In TBO algorithm, the landscape is divided into multiple partitions in every iteration. Partitioning the landscape allows search to perform more accurately and scrupulously in every part. In other words, the search sub-algorithm focuses on a smaller part to use its all ability for searching a part and not the entire landscape. 
On the other hand, when not using TBO algorithm, the search algorithm faces a big and perhaps high dimensional search space just by itself! It tries its best to find the optima however the landscape is too high and big for it to search accurately enough. In order to make the search algorithm more powerful for searching the whole landscape, the numbers of particles should be increased causing the algorithm to search very slowly because of the more iterations of particles.

\begin{figure}[!t]
\centering
\includegraphics[width=3in]{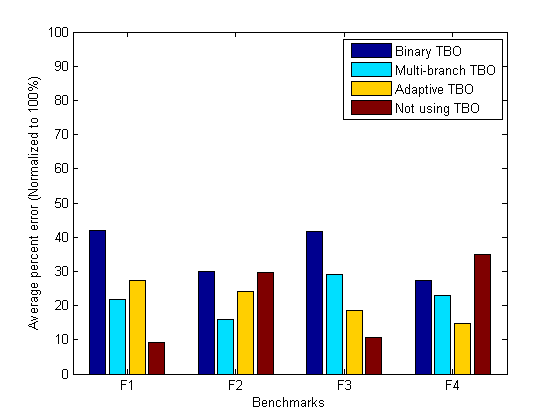}
\caption{Average errors (in percents) of Table \ref{table_test}.}
\label{bar}
\end{figure}

\begin{figure}[!t]
\centering
\includegraphics[width=3in]{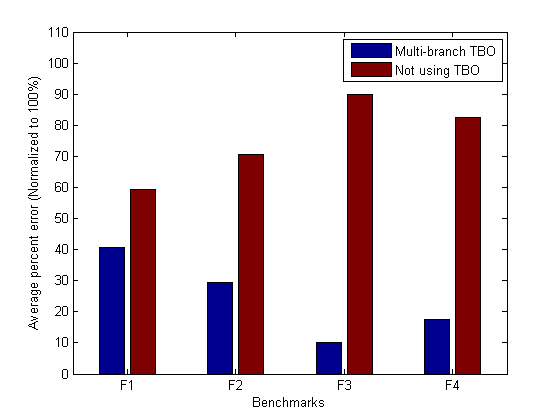}
\caption{Average errors (in percents) of Table \ref{table_test_3D}.}
\label{bar_for_3D}
\end{figure} 

\section{Conclusion and Future Direction}\label{section_conclusion}
This paper proposed a meta-algorithm, called Tree-Based Optimization (TBO), which uses other metaheuristic optimization algorithms as its sub-algorithm to enhance the performance of the search operation. 
TBO algorithm tries to remove the parts of search space that have low fitness in order to limit the search space to the parts really worthy to search. This algorithm minimizes the landscape level by level until it converges to the global optima. Three different types of TBO algorithm were proposed and explained in detail and were tested for different benchmarks. Results showed wonderfully good performance of TBO algorithm especially in high dimensional landscapes.
This algorithm performs much better in high dimensional landscapes since it simplifies the search exponentially and improves the search both in accuracy and speed.
TBO algorithm might be extended to classic mathematical optimization. It can use methods such as Gradient Descent as its sub-algorithm in order to search for the optima in the divided regions. Using classical methods in TBO algorithm might be a very powerful optimization method, which will be deferred as a future work.


%

\appendices




\ifCLASSOPTIONcaptionsoff
  \newpage
\fi



%

\bibliographystyle{IEEEtran}
\bibliography{References}

%








\end{document}